\documentclass[sigconf]{acmart}

\usepackage{adjustbox}
\usepackage{amsthm}
\usepackage[linesnumbered,ruled,vlined]{algorithm2e}
\usepackage[inline, shortlabels]{enumitem}
\usepackage{graphicx}
\usepackage{hyperref}
\usepackage{cleveref} 
\usepackage{subcaption}
\usepackage[HTML]{xcolor}
\usepackage{balance}

\theoremstyle{plain}
\newtheorem{example}{Example}
\newtheorem{theorem}{Theorem}
\newtheorem{corollary}{Corollary}

\usepackage{colortbl}
\definecolor{bestcell}{HTML}{D6E6F2}   
\newcommand{\best}[1]{\cellcolor{bestcell}\textbf{#1}}

\definecolor{lightblue}{HTML}{C8E7FA}
\definecolor{lightorange}{HTML}{FBD6B7}

\DeclareMathOperator*{\argmax}{arg\,max}
\DeclareMathOperator{\E}{\mathbb{E}}

\def\showcomments{1}
\newcommand{\KH}[1]{\if\showcomments1\textcolor{violet}{(KH: #1)}\fi}
\newcommand{\CZ}[1]{\if\showcomments1\textcolor{blue}{(CZ: #1)}\fi}
\newcommand{\BS}[1]{\if\showcomments1\textcolor{orange}{(BS: #1)}\fi}

\AtBeginDocument{%
  }

\setcopyright{acmlicensed}
\copyrightyear{2026}
\acmYear{2026}
\setcopyright{cc}
\setcctype{by}
\acmConference[KDD '26]{Proceedings of the 32nd ACM SIGKDD Conference on Knowledge Discovery and Data Mining V.2}{August 09--13, 2026}{Jeju Island, Republic of Korea}
\acmBooktitle{Proceedings of the 32nd ACM SIGKDD Conference on Knowledge Discovery and Data Mining V.2 (KDD '26), August 09--13, 2026, Jeju Island, Republic of Korea}
\acmDOI{10.1145/3770855.3817882}
\acmISBN{979-8-4007-2259-2/2026/08}

\begin{document}

\title{Adaptive Exploration for Latent-State Bandits}

\author{Jikai Jin}
\orcid{0009-0004-7988-4650}
\affiliation{
  \department{The Institute for Computational and Mathematical Engineering}
  \institution{Stanford University}
  \city{Stanford}
  \state{CA}
  \country{USA}
}
\email{jkjin@stanford.edu}

\author{Kenneth Hung}
\orcid{0000-0002-3911-182X}
\affiliation{
  \department{Ads Online Experimentation}
  \institution{Meta Platforms, Inc.}
  \city{Menlo Park}
  \state{CA}
  \country{USA}
}
\email{kenhung@meta.com}

\author{Sanath Kumar Krishnamurthy}
\orcid{0009-0006-1197-5358}
\affiliation{
  \department{Ranking AI Research}
  \institution{Meta Platforms, Inc.}
  \city{Menlo Park}
  \state{CA}
  \country{USA}
}
\email{sanathsk@meta.com}

\author{Baoyi Shi}
\orcid{0009-0009-7362-0714}
\affiliation{
  \department{Central Applied Science}
  \institution{Meta Platforms, Inc.}
  \city{Menlo Park}
  \state{CA}
  \country{USA}
}
\email{baoyis@meta.com}

\author{Congshan Zhang}
\orcid{0009-0005-2815-7047}
\affiliation{
  \department{Ads Online Experimentation}
  \institution{Meta Platforms, Inc.}
  \city{Menlo Park}
  \state{CA}
  \country{USA}
}
\email{cszhang@meta.com}

\renewcommand{\shortauthors}{Jin et al.}

\begin{abstract}
    We study bandits whose rewards depend on an unobserved Markov state that evolves independently of the learner's actions. The optimal arm can change even though the learner observes only past actions and rewards. We propose algorithms that feed LinUCB with two summaries of the hidden state: a lagged action-reward pair and, when available, a probe fingerprint formed from rewards of multiple arms. The adaptive variants refresh the fingerprint using residual, margin, and staleness tests. In synthetic stress tests over state count, transition rate, noise, and horizon, these methods reduce dynamic regret relative to standard, adversarial, and non-stationary bandit baselines when the summaries distinguish states and are updated often enough. Ablations and misspecification tests identify the main failure modes: weak fingerprint separation, high noise, and state changes during sequential probes.

\end{abstract}




\maketitle

\section{Introduction}
The multi-armed bandit problem models sequential decision-making with applications in online advertising \citep{chen2013combinatorial}, recommendation systems \citep{li2010contextual}, clinical trials \citep{durand2018contextual}, and resource allocation \citep{henderson2022beyond}. Classical bandit algorithms typically assume that reward distributions are constant or depend only on observable contextual features. We study a non-stationary setting with hidden time-varying states that
\begin{enumerate*}
    \item significantly influence the action-reward mapping,
    \item transition autonomously with no influence from the action taken, but
    \item remain unobservable to the decision-maker.
\end{enumerate*}
In this setting, the reward-maximizing action depends on latent environmental conditions inferred from past observations. \Cref{fig:dag} shows a directed acyclic graph illustrating the relationship among the hidden states, actions, and rewards.

\begin{figure}
    \centering
    \includegraphics[width=0.85\linewidth]{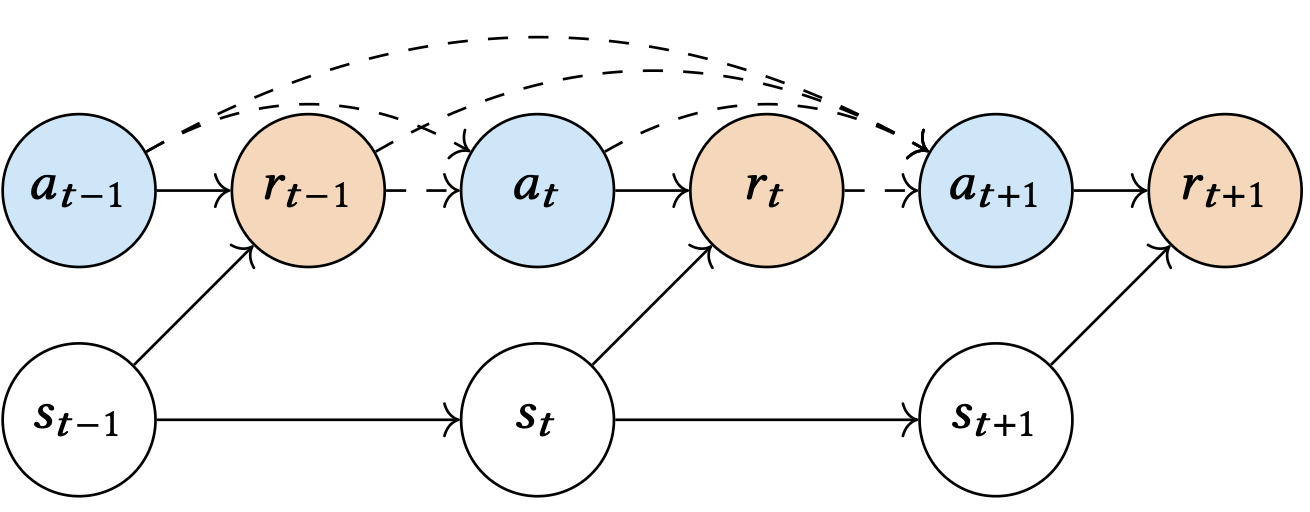}
    \Description[A directed acyclic graph (DAG) representing the causal relationship among hidden state $s_t$, action $a_t$ and reward $r_t$.]{A directed acyclic graph (DAG) representing the causal relationship among hidden state $s_t$, action $a_t$ and reward $r_t$, the relationship between action and reward as in usual bandit literature where $a_t \to r_t$, with the feedback from past actions and rewards to the current action $a_t$. However, state $s_t$ also affects $r_t$, and evolves autonomously as a Markov chain.}
    \caption{A directed acyclic graph (DAG) representing the causal relationship among hidden state $s_t$, action $a_t$ and reward $r_t$. The solid arrows represent the latent-state bandit setting we have, while the dashed arrows represent the feedback from past actions and rewards due to the bandit algorithm.}
    \label{fig:dag}
\end{figure}

As a motivating example, consider an online advertising platform that must repeatedly select which ad to display. While the platform observes user demographics and browsing history, it cannot directly observe broad environmental variables (e.g., the weather) or the user's current cognitive state (e.g., whether they are focused or distracted, browsing casually or ready to purchase). These hidden states affect ad effectiveness and evolve over time: a user might transition from ``browsing'' to ``purchase-ready'', or their attention might fluctuate based on the weather. An algorithm that ignores these latent dynamics will misestimate ad performance and make suboptimal decisions.

The example highlights a limitation of standard contextual bandits: they condition on observed features, whereas the relevant state variables may be unobserved. This leads to \textit{systematic confounding}, where the algorithm's reward estimates are biased by unobserved factors, resulting in sub-optimal long-term performance.

\subsection{Key Challenges}
The presence of unobserved time-varying confounders creates three technical challenges:

\begin{enumerate}
    \item Confounded reward estimation: Standard bandit algorithms estimate arm values by averaging observed rewards, but when rewards depend on hidden states, these estimates reflect a mixture over state distributions rather than state-conditional values. This leads to systematic bias when the optimal action varies across states --- standard algorithms will converge to an arm that is mediocre in all states but optimal in none.
    \item Temporal reward attribution: When rewards depend on hidden states that evolve over time, observed performance changes may result from either learning progress or state transitions. Disentangling these effects complicates the exploration exploitation trade-off: a reward increase after an exploratory action may reflect either parameter learning or a favorable state transition, while a reward decrease may reflect the reverse.
    \item Limited state information: Unlike fully observable contextual bandits, algorithms cannot directly condition on state information and must rely on indirect signals to infer the current state. This task becomes increasingly difficult as the state space grows.
\end{enumerate}

\subsection{Our Contributions}
Hidden states leave traces in the observed reward sequence. We use these traces to build low-dimensional statistics for state-dependent arm selection, while avoiding explicit estimation of the latent Markov chain. Concretely, develop a family of bandit algorithms that track the dynamically optimal arm without estimating the Markov transition model, provided that the resulting statistics separate the relevant states and remain current. Our approach combines two complementary strategies: 
\begin{enumerate}
    \item Lagged contextual learning: Previous rewards and actions serve as informative proxies for hidden states. Our \textbf{Lagged-context UCB (LC-UCB)} algorithm treats the action-reward pair $(r_{t-1},a_{t-1})$ as contextual features, enabling implicit state tracking via observable interaction history.
    \item Coordinated probing: Motivated by an arm identification issue of LC-UCB (see \Cref{ex:state-ambiguity}), we develop probing strategies that generate contextual signals by regularly exploring different arms in coordinated patterns. This strategy reveals cross-arm relationships and reward structures that would remain hidden under LC-UCB. We introduce \textbf{Randomized Probing UCB (RP-UCB)} for settings with multiple experimental units that can be simultaneously assigned different arms, \textbf{Sequential Probing UCB (SP-UCB)} for single-unit settings, using sequential arm sampling.
\end{enumerate}

These two approaches can be combined: lagged context provides continuous state tracking, while coordinated probing supplies the joint information needed to disambiguate state-dependent reward patterns. We therefore develop adaptive probing schemes that refresh probe information only when the current summary appears unreliable, leading to \textbf{Adaptive Randomized Probing UCB (AdaRP-UCB)} and \textbf{Adaptive Sequential Probing UCB (AdaSP-UCB)}.

The algorithms keep the classical bandit interface: they construct features from observed interaction data and pass them to a contextual learner. This keeps computation close to LinUCB while adding state information needed for structured non-stationarity.

We conduct systematic empirical evaluation across parameter sweeps --- varying the number of hidden states, environmental volatility, reward noise, and time horizons. The adaptive probing methods improve over classical, adversarial, and non-stationary bandit baselines when the summaries remain discriminative and current, with the largest gains in environments where the optimal arm changes often enough to make timely updates valuable.

\subsection{Related Work}

\paragraph{Classical and Contextual Bandits}
The multi-armed bandit problem, introduced by \citet{robbins1952some}, includes foundational algorithms like UCB1 \citep{auer2002finite} and Thompson Sampling \citep{thompson1933likelihood,russo2018tutorial}. These classical methods assume stationary reward distributions, which are violated in our setting. Contextual bandits \citep{li2010contextual,beygelzimer2011contextual} like LinUCB require all relevant features to be observable. Recent work on bandits with unobserved confounders \citep{bareinboim2015bandits,forney2017counterfactual,liao2024instrumental} typically assumes access to instrumental variables or structural knowledge unavailable here.

\paragraph{Doubly-robust approaches}
While intended to reduce bias, doubly-robust approaches \citep{dimakopoulou2019balanced,kim2021doubly} also sever feedback-induced confounding ($s_{t-1} \to r_{t-1} \to a_t$ and $s_{t-1} \to s_t \to r_t$ in \Cref{fig:dag}). However, they often introduce high variance through inverse propensity weights and estimate rewards averaged over a state distribution rather than adapting to the current hidden state.

\paragraph{Non-stationary and adversarial bandit}
Non-stationary algorithms use sliding windows or discounting \citep{besbes2019optimal,garivier2011upper,cao2019nearly} to treat changes as arbitrary drift. Semiparametric variants \citep{krishnamurthy2018semiparametric} allow non-stationarity but lack the flexibility to handle shifts in arm rankings or gaps. Adversarial algorithms like EXP3 \citep{auer2002nonstochastic} target the single retrospective best arm, which leads to overly conservative exploration when aiming for dynamic regret. We focus instead on a structured Markovian state determining the reward mapping.

\paragraph{Restless bandit and reinforcement learning}
Restless bandits \citep{whittle1988restless,liu2010indexability} model Markovian dynamics, but assume observable states and known, action-dependent transitions. In reinforcement learning, Partially observable MDPs (POMDPs) \citep{kaelbling1998planning,pineau2006anytime} assume decodability \citep{azizzadenesheli2016reinforcement,jin2020sample,efroni2022provable,xie2024learning} to distinguish hidden states, which does not hold here. Furthermore, POMDP transitions are action-dependent, whereas our state evolution is autonomous. Contextual bandits in Markov environments \citep{hallak2015contextual,modi2018markov} or information-theoretic approaches \citep{hao2022regret,arumugam2022deciding} similarly require full observability or sufficient statistics.

\paragraph{Latent Markov decision processes}
Highly relevant is work on Latent Markov decision processes (LMDPs) \citep{kwon2021reinforcement,kwon2023reward,kwon2024rl}. However, LMDPs assume an episodic structure with contexts fixed throughout an episode. Our setting is a single continuous interaction sequence: the hidden state may change within the sequence, and the learner is not given a sufficient statistic. The algorithms use lightweight history and probe summaries in place of a learned latent transition model. We summarize these differences in \Cref{tbl:comparison}.

\section{Setup and Notations}
\paragraph{Latent-state bandit model}

We consider a non-stationary multi-armed bandit problem with $K$ arms, $S$ hidden states, and a time horizon of $T$ rounds. Let $\mathcal S$ denote the hidden-state set and index arms by $\mathcal A=\{0,\ldots,K-1\}$ throughout. At each round $t = 1, 2, \ldots, T$, the environment occupies an unobservable hidden state $s_t \in \mathcal S$. The hidden state evolves according to a Markov chain with transition matrix $\mathbf{P}$, where $P_{ij} = \mathbb{P}(s_{t+1} = j \mid s_t = i)$ denotes the probability of transitioning from state $i$ to state $j$. We assume the chain is ergodic with stationary distribution $\boldsymbol{\pi}^*$.

Upon selecting an arm $a_t \in \mathcal A$, the agent observes a stochastic reward
\[
    r_t = \mu_{s_t, a_t} + \eta_t,
\]
where $\mu_{s,a} \in [0,1]$ is the mean reward of arm $a$ in state $s$ and $\eta_t$ is mean zero with variance $\sigma^2$. The agent never observes the state $s_t$ directly --- only the scalar reward $r_t$.

\paragraph{Dynamic regret}

For each state $s$, define an optimal arm $a_s^*\in\argmax_{a\in\mathcal{A}}\mu_{s,a}$ and gaps $\Delta_{s,a} \coloneq \mu_{s,a_s^*}-\mu_{s,a}\ge 0$. Let $\Delta_{\max} \coloneq \max_{s\in\mathcal{S},\,a\in\mathcal{A}}\Delta_{s,a}\le 1$ denote the maximum single-step regret. The (single-unit) dynamic regret against the state-aware oracle is
\[
    R_T \coloneq \sum_{t=1}^T\Delta_{s_t,a_t}.
\]
For any history statistic $Z_t$, let $\pi_t^Z(Z_t)$ be the arm chosen by the best rule that observes only $Z_t$. The regret of an algorithm $A$ then decomposes as
\[
    R_T(A)
    =
    \underbrace{\sum_{t=1}^T
    \big(\mu_{s_t,a^*_{s_t}}-\mu_{s_t,\pi_t^Z(Z_t)}\big)}_{\mathrm{App}_T(Z)}
    +
    \underbrace{\sum_{t=1}^T
    \big(\mu_{s_t,\pi_t^Z(Z_t)}-\mu_{s_t,a_t}\big)}_{\mathrm{Reg}_T^Z(A)}.
\]
Here $\mathrm{App}_T(Z)$ is the regret of this representation oracle against the state-aware oracle, while $\mathrm{Reg}_T^Z(A)$ is the learner's regret against the representation oracle. This separates state ambiguity from learner regret before introducing lagged context and probing.

\paragraph{Notations}

Our algorithms construct context features from past observations to enable state inference. We write $\boldsymbol{\phi}_t \in \mathbb{R}^d$ for the context vector at round $t$, where the specific construction depends on the algorithm. For LC-UCB, we could use the lagged context $\boldsymbol{\phi}_t = (a_{t-1}, r_{t-1})$, concatenating the previous action\footnote{or, where appropriate, a one-hot encoding thereof to apply LinUCB} with the previous reward. For probing algorithms (RP-UCB and its variants), we additionally incorporate joint observations $(r_0^{\mathrm{prev}}, r_1^{\mathrm{prev}})$ obtained by sampling both arms.

All algorithms build upon the LinUCB algorithm \citep{li2010contextual}. For each arm $a\in\mathcal A$, we maintain a precision matrix $\mathbf{A}_a \in \mathbb{R}^{d \times d}$, a reward-weighted feature vector $\mathbf{b}_a \in \mathbb{R}^d$, and compute parameter estimates $\hat{\boldsymbol{\theta}}_a = \mathbf{A}_a^{-1} \mathbf{b}_a$. Arm selection follows the upper confidence bound principle: at each round, we select $a_t = \arg\max_{a\in\mathcal A} \mathrm{UCB}_{t,a}$, where $\mathrm{UCB}_{t,a} = \boldsymbol{\phi}_t^\top \hat{\boldsymbol{\theta}}_a + \alpha \sqrt{\boldsymbol{\phi}_t^\top \mathbf{A}_a^{-1} \boldsymbol{\phi}_t}$ and $\alpha > 0$ controls exploration. The regularization parameter $\lambda > 0$ initializes the precision matrices as $\mathbf{A}_a = \lambda \mathbf{I}$.

For fixed-schedule probing, the parameter $\tau$ specifies the probing frequency. For adaptive probing, we use thresholds $z_{\mathrm{thresh}}$ (residual gate), $m_{\mathrm{thresh}}$ (uncertainty gate), and $(\lambda_h, \delta_h)$ (hazard gate), along with a minimum inter-probe gap $\tau_{\min}$. \Cref{tbl:notation} summarizes the key symbols for reference.

\begin{table}[hbtp]
\centering
\begin{tabular}{cl}
\toprule
\textbf{Symbol} & \textbf{Description} \\
\midrule
$K$ & Number of arms \\
$S$ & Number of hidden states \\
$T$ & Time horizon \\
$s_t$ & Hidden state at round $t$ \\
$a_t$ & Arm selected at round $t$ \\
$r_t$ & Reward at round $t$ \\
$\mu_{s,a}$ & Expected reward (state $s$, arm $a$) \\
$\mathbf{P}$ & State transition matrix \\
$\boldsymbol{\phi}_t$ & Context vector at round $t$ \\
$\mathbf{A}_a$ & Precision matrix for arm $a$ \\
$\hat{\boldsymbol{\theta}}_a$ & Parameter estimate for arm $a$ \\
$\mathrm{UCB}_{t,a}$ & Upper confidence bound \\
$\alpha$ & Exploration parameter \\
$\lambda$ & Regularization parameter \\
$\tau$ & Probing frequency \\
$\tau_{\min}$ & Minimum inter-probe gap \\
\bottomrule
\end{tabular}
\vspace{4pt}
\caption{Summary of key notation.}
\label{tbl:notation}
\end{table}

\vspace{-30pt}

\section{Lagged Action-Reward as Context}
The main challenge in our setting is that reward distributions depend on hidden states that evolve over time. Since direct state observation is impossible, we must develop algorithms that infer state information indirectly from observable quantities. This section introduces LC-UCB in \Cref{alg:lc-ucb}, which utilizes a simple but powerful insight: under Markovian state dynamics, recent rewards carry information about the current hidden state.

Consider the information available to the agent at time $t$. While the current state $s_t$ is unobservable, the agent has access to the previous action $a_{t-1}$ and the resulting reward $r_{t-1}$. Since the reward $r_{t-1}$ was generated from a conditional distribution fully determined by $a_{t-1}$ and $s_{t-1}$, the previous action-reward pair $(a_{t-1}, r_{t-1})$ provides information about $s_{t-1}$. If the hidden state evolves as a Markov chain slowly and the mixing time is long, the previous state $s_{t-1}$ is highly informative of the current state $s_t$. Therefore, the previous action-reward pair $(a_{t-1}, r_{t-1})$ acts as an informative proxy about the current state $s_t$, and we can use the pair as a contextual feature and applying contextual bandit algorithms to learn state-dependent reward mappings.

\begin{algorithm}[hbtp]
\caption{Lagged-Context Upper Confidence Bound (LC-UCB)}
\label{alg:lc-ucb}
\KwData{$\alpha > 0$ (exploration parameter), $\lambda > 0$ (regularization), $T$ (time horizon), $K$ (number of arms)}
\For{$a \gets 0, 1, \ldots, K-1$}{
    Initialize LinUCB model $\mathcal{M}_a$ with exploration parameter $\alpha$ and regularization $\lambda$\;
}
$(a_0, r_0) \gets (0, 0)$\tcp*{Initialize previous action and reward}
\For{$t \gets 1, 2, \ldots, T$}{
    $\boldsymbol{\phi}_t \gets (a_{t-1}, r_{t-1})$\;
    \For{$a \gets 0, 1, \ldots, K-1$}{
        Compute $\mathrm{UCB}_{t,a}$ using LinUCB model $\mathcal{M}_a$ with context $\boldsymbol{\phi}_t$\;
    }
    $a_t \gets \argmax_{a\in\mathcal A} \text{UCB}_{t,a}$\tcp*{Select arm}
    Play action $a_t$ and observe reward $r_t$\;
    Update LinUCB model $\mathcal{M}_{a_t}$ with $(\boldsymbol{\phi}_t, r_t)$\;
}
\end{algorithm}

\subsection{Intuition}

LC-UCB differs from UCB1 in the target it estimates.

Specifically, UCB1 maintains empirical means
\[
    \hat\mu_{t,a} = \frac{1}{N_{t,a}} \sum_{i=1}^t r_i \mathbf{1}_{a_i = a}, \quad \text{where } N_{t,a} = \#\{i: 1 \le i \le t \text{ and } a_i = a\},
\]
for each arm $a$, which simply neglects the existence of an evolving state $s_t$. Since the hidden state Markov chain is ergodic, the state visiting distribution converges to some $\pi_{\mathrm{UCB}}^*$. So as $t \to \infty$, the empirical means $\hat\mu_{a,t}$ should converge to $\E_{a \sim \pi_{\mathrm{UCB}}^*(s)} \mu_{s,a}$ and UCB1 effectively solves
\[
    a^{\mathrm{UCB1}} = \argmax_{a\in\mathcal A} \E_{a \sim \pi_{\mathrm{UCB}}^*(s)} \mu_{s,a}.
\]
This finds the optimal arm under a certain state distribution, which may be suboptimal in individual states.

By contrast, for each arm $a$, LC-UCB collects data
\[
    \mathcal{D}_a = \{(a_{t-1}, r_{t-1}, r_t): a_t = a\}.
\]
The algorithm proceeds to learn a mapping from the context $(a_{t-1},\allowbreak r_{t-1})$ to $r_t$. Now
\begin{multline}
    \E[r_t \mid a_{t-1} = a', r_{t-1} = r'] \\
    = \sum_s \mathbb{P}[s_t = s \mid a_{t-1} = a', r_{t-1} = r'] \cdot \mu_{s,a}.
\label{eq:expand-cond-expectation}
\end{multline}
For a Markov chain with slow mixing time, we would expect
\[
    \mathbb{P}[s_t = s \mid a_{t-1} = a', r_{t-1} = r'] \approx \mathbb{P}[s_{t-1} = s \mid a_{t-1} = a', r_{t-1} = r'].
\]
Now by Bayes rule, we have
\begin{multline*}
    \mathbb{P}[s_{t-1} = s \mid a_{t-1} = a', r_{t-1} = r'] \\
    \propto \pi_{\mathrm{LC-UCB}}^*(s, a') \cdot \mathbb{P}[r_{t-1} = r' \mid s_{t-1} = s, a_{t-1} = a']
\end{multline*}
where $\pi_{\mathrm{LC-UCB}}^*(s', a')$ is the algorithm-dependent probability that each pair $(s', a')$ is visited.

When rewards are separated and $\pi_{\mathrm{LC-UCB}}^*(s', a')$ assigns non-negligible probability to the relevant state-action pairs, Bayes' rule places most posterior mass on the recent state $s$ whose mean reward for arm $a'$ is close to $r'$. Combining this with \eqref{eq:expand-cond-expectation} gives
\[
    \E[r_t \mid a_{t-1} = a', r_{t-1} = r'] \approx \mu_{s,a}.
\]
Thus this conditional expectation approximates the reward mean in the inferred recent state. When these conditional means are well estimated for all $a$, LC-UCB selects the arm with the largest $\mu_{s,a}$ for that inferred state.

Let $\gamma_a := \min_{s\ne s'}|\mu_{s,a}-\mu_{s',a}|$ and $\beta_1 := \sup_s(1-\max_u P_{su})$. Lagged context is useful when $\gamma_{a_{t-1}}$ is large relative to $\sigma$ and $\beta_1$ is small: the reward then identifies a recent state, and the transition matrix makes the current state predictable from it. If two states give the same distribution of $(a_{t-1},r_{t-1})$ and have different optimal arms, any policy that only sees that statistic has unavoidable approximation error. \Cref{ex:state-ambiguity} illustrates this obstruction and motivates probing.

\section{Probing for State Fingerprint}
LC-UCB can fail when multiple states give similar rewards for the same arm, because the lagged context may not distinguish them\footnote{Similar to the statistical notion of identifiability}. Probing addresses this ambiguity by collecting multi-arm reward summaries, which we call \textit{state fingerprints}.

\begin{example}
\label{ex:state-ambiguity}
Consider the illustrative example shown in \Cref{tbl:state-ambiguity}, where each entry is the mean reward under the given state and arm. When LC-UCB observes the context $(a_{t-1} = 0, r_{t-1} = 0.4)$, this observation is consistent with two possible states: states 0 and 1. Critically, these states have \emph{different} optimal arms: arm 0 is optimal in state 0 (since $0.4 >0.3$), while arm 1 is optimal in state 1 (since $0.5>0.4$). Thus, the lagged context alone can be non-identifying in such scenarios, motivating a probing approach.

\begin{table}[htbp]
\centering
\begin{tabular}{c c c c c}
\toprule
 & \textbf{State 0} & \textbf{State 1} & \textbf{State 2} & \textbf{State 3} \\
\midrule
\textbf{Arm 0} & 0.4 & 0.4 & 0.6 & 0.6 \\
\textbf{Arm 1} & 0.3 & 0.5 & 0.5 & 0.3 \\
\midrule
\textbf{Fingerprint $(r_0, r_1)$} & (0.4, 0.3) & (0.4, 0.5) & (0.6, 0.5) & (0.6, 0.3) \\
\bottomrule
\end{tabular}
\vspace{2pt}
\caption{Reward Matrix in \Cref{ex:state-ambiguity}.}
\label{tbl:state-ambiguity}
\end{table}
\end{example}

\vspace{-20pt}

This ambiguity becomes more likely as the number of states $S$ grows, because reward overlaps across states become more common. Single-arm observations can then be uninformative about which arm is optimal in the underlying state, so LC-UCB may degrade in environments with large state spaces and complex reward patterns.

One way to reduce this ambiguity is to observe rewards from multiple arms in the same state. If we sample both arms, each state now has a unique fingerprint --- a pair of rewards that jointly identify the state even when individual rewards do not. In \Cref{tbl:state-ambiguity}, states 0 and 1, previously indistinguishable, are now separated: $(0.4, 0.3)$ vs $(0.4, 0.5)$.

The same idea appears in randomized experiments \citep{rubin1974estimating, kohavi2009controlled}: different units can receive different treatments at the same time. In our bandit setting, assigning different arms to concurrent units yields same-state reward pairs $(r_0, r_1)$ that help identify the latent state.

\subsection{Two Probing Paradigms}
We develop two algorithms based on how fingerprints are collected, presented in \Cref{alg:rp-ucb,alg:sp-ucb}. The fingerprint representation extends to larger arm sets, but the single-unit sequential version becomes less attractive as the probe window grows.

\begin{enumerate}

\item Randomized Probing UCB (RP-UCB): Some applications naturally support simultaneous experimentation: online platforms can randomize different users to different recommendation algorithms \citep{agarwal2009explore,tang2013automatic}, clinical trials can assign different patients to different treatment protocols \citep{berry2012adaptive,villar2015multi,lai2014adaptive}, and financial trading systems can allocate portions capital to different strategies \citep{helmbold1998line}. In these settings, we can obtain true joint observations $(r_0, r_1)$ from the same time period.

\item Sequential Probing UCB (SP-UCB): When only a single decision-making unit exists, simultaneous probing is impossible. Instead, we sample arms consecutively --- playing arm $0$ at time $t$, then arm $1$ at time $t+1$, and treat $(r_t, r_{t+1})$ as a synthetic fingerprint. This approach relies on slow state transitions: if the hidden state is sufficiently ``sticky'' (i.e., with high self-transition probability), consecutive observations approximate true joint observations.
\end{enumerate}

A probe records one fingerprint, but RP-UCB and SP-UCB pay different costs as $K$ grows. RP-style probing adds coordinates to a same-round fingerprint; its statistical requirement is joint separation of the state fingerprints. SP-style probing uses $K$ consecutive rounds, so both probe cost and staleness grow with $K$. 

\subsection{Feature Construction}
\label{sec:feature-construction}

Both algorithms combine fingerprint information with lagged context to form the feature vector
\begin{equation}
\label{eq:combined-features}
\boldsymbol{\phi}_t \coloneq \phi_{\mathrm{fp}} \oplus \phi_{\mathrm{lag}} = (r_0^{\mathrm{fp}}, r_1^{\mathrm{fp}}) \oplus (a_{t-1}, r_{t-1})
\end{equation}
where $\oplus$ denotes vector concatenation. The two components are defined as:
\begin{enumerate}
    \item Fingerprint features: $\phi_{\mathrm{fp}} \coloneq (r_0^{\mathrm{fp}}, r_1^{\mathrm{fp}})$ captures the most recent joint observation from probing.
    
    \item Lagged features: $\phi_{\mathrm{lag}} \coloneq (a_{t-1}, r_{t-1})$ captures the previous action\footnote{or, where appropriate, a one-hot coding thereof} $a_{t-1}$ and the previous reward $r_{t-1}$.
\end{enumerate}

A fingerprint remains fixed between probes and may become stale after a state transition. The lagged pair $(a_{t-1}, r_{t-1})$ changes every round and adds a more recent, though noisier, state signal.

\begin{algorithm}[htbp]
\caption{Randomized Probing Upper Confidence Bound (RP-UCB)}
\label{alg:rp-ucb}
\KwData{$\alpha$, $\lambda$, $\tau$ (probing interval), $T$, $K \gets 2$}
\For{$a \gets 0, 1$}{
    Initialize LinUCB model $\mathcal{M}_a$ with exploration parameter $\alpha$ and regularization $\lambda$
}
$(r_0^{\mathrm{fp}}, r_1^{\mathrm{fp}}) \gets (0, 0)$\tcp*{Initialize fingerprint}
$(a_0, r_0) \gets (0, 0)$ \tcp*{Init. prev. action and reward}
\For{$t \gets 1, 2, \ldots, T$}{
    $\boldsymbol\phi_{\mathrm{fp}} \gets (r_0^{\mathrm{fp}}, r_1^{\mathrm{fp}})$\tcp*{Fingerprint features}
    $\boldsymbol\phi_{\mathrm{lag}} \gets (a_{t-1}, r_{t-1})$\tcp*{Lagged features}
    $\boldsymbol{\phi}_t \gets \boldsymbol\phi_{\mathrm{fp}} \oplus \boldsymbol\phi_{\mathrm{lag}}$\;
    \eIf{$t \bmod \tau = 0$}{
        $\mathtt{mode} \gets \text{probe}$\tcp*{Probe mode}
        $(a_t^C, a_t^T) \gets (0, 1)$\;
    }{
        $\mathtt{mode} \gets \text{exploit}$\tcp*{Exploit mode}
        \For{$a \gets 0, 1$}{
            Compute $\mathrm{UCB}_{t,a}$ using LinUCB model $\mathcal{M}_a$ with context $\boldsymbol{\phi}_t$\;
        }
        $a^* \gets \argmax_{a\in\{0,1\}} \mathrm{UCB}_{t, a}$\;
        $(a_t^C, a_t^T) \gets (a^*, a^*)$\;
    }
    Play actions $(a_t^C, a_t^T)$ and observe rewards $(r_t^C, r_t^T)$\;
    Update LinUCB models $\mathcal{M}_{a_t^C}$ and $\mathcal{M}_{a_t^T}$ with $(\boldsymbol{\phi}_t, r_t^C)$ and $(\boldsymbol{\phi}_t, r_t^T)$\;
    \eIf{$\mathtt{mode} = \text{probe}$}{
        $(r_0^{\mathrm{fp}}, r_1^{\mathrm{fp}}) \gets (r_t^C, r_t^T)$\;
        \eIf{$r_t^T > r_t^C$}{
            $(a_t, r_t) \gets (a_t^T, r_t^T)$\;
        }{
            $(a_t, r_t) \gets (a_t^C, r_t^C)$\;
        }
    }{
        $(a_t, r_t) \gets (a_t^C, (r_t^C + r_t^T) / 2)$\;
    }
}
\end{algorithm}

The probing interval trades off direct probe regret, probe-time identification error, and optimal-arm drift after a fingerprint is reused. The next result states this trade-off for a generic sequence of probe episodes, covering both single-round and multi-round probe blocks.

\begin{theorem}[Probe-episode regret decomposition]
\label{thm:probe-episode}
Suppose a policy's execution consists of $M$ completed probe episodes and a leftover set $\mathcal U$ of initialization or trailing rounds. In episode $m$, a probe block $\mathcal P_m$ is completed at time $c_m$ and is followed by $L_m-1$ exploit rounds, indexed by lags $j=1,\ldots,L_m-1$. During these exploit rounds the policy reuses the arm $\hat a_m^\ast$ estimated from the completed probe block. Let $\mathcal H_{m,j}$ be the observable history immediately before the lag-$j$ exploit round, including the event that this round occurs. Suppose the probe block has conditional expected regret at most $C_m$ given the history before the block, and that for every reached lag $j<L_m$,
\begin{enumerate}
    \item[(i)] $\mathbb{P}\!\left(\hat{a}_m^\ast \neq a^\ast_{s_{c_m}} \mid \mathcal H_{m,j}\right) \le \varepsilon_m$;
    \item[(ii)] $\mathbb{P}\!\left(a^\ast_{s_{c_m+j}} \neq a^\ast_{s_{c_m}} \mid \mathcal H_{m,j}\right) \le \nu_j$.
\end{enumerate}
Then
\begin{equation}
\label{eq:probe-episode}
    \mathbb{E}[R_T]
    \;\le\;
    \mathbb{E}\!\left[\sum_{m=1}^M C_m\right]
    \;+\; \Delta_{\max} \cdot \mathbb{E}\!\left[\sum_{m=1}^M \sum_{j=1}^{L_m - 1}\!\big(\varepsilon_m + \nu_j\big)\right]
    \;+\; \Delta_{\max}\,\mathbb{E}|\mathcal U|.
\end{equation}
\end{theorem}

\Cref{thm:probe-episode} clarifies the role of the adaptive gates introduced in the next section: the hazard gate upper-bounds the episode length $L_m$ and therefore caps the drift sum, while the residual and uncertainty gates target the probe-time error $\varepsilon_m$. In particular, we have the following regret bound for the periodic probing scheme:

\begin{corollary}[Periodic probing]
\label{thm:tau}
Suppose $S = 2$ and each state has a unique optimal arm, and that the Markov chain satisfies $q := \sup_{s} \mathbb{P}(s_{t+1} \neq s_t \mid s_t = s)$, so that $\nu_j \le j\, q$. Consider an idealized policy that probes once every $\tau$ rounds, with each probe block incurring regret at most $\Delta_{\mathrm{probe}}$ and returning an estimate of the current optimal arm with error probability at most $\varepsilon_{\mathrm{fp}}$. Then
\[
    \frac{1}{T}\,\mathbb{E}[R_T]
    \le
    \frac{\Delta_{\mathrm{probe}}}{\tau}
    + \frac{\Delta_{\max}\, q\, (\tau-1)}{2}
    + \Delta_{\max}\,\varepsilon_{\mathrm{fp}}
    + O\!\left(\frac{\Delta_{\mathrm{probe}}+\tau\Delta_{\max}}{T}\right).
\]
Ignoring $\varepsilon_{\mathrm{fp}}$ and lower-order endpoint terms, the bound is minimized at $\tau^* \asymp \sqrt{\Delta_{\mathrm{probe}}/(\Delta_{\max}\, q)}$; a more volatile chain (larger $q$) calls for more frequent probing.
\end{corollary}

The proof can be found in \Cref{app:probe-schedule-proof}.

\paragraph{When do lagged context and fingerprints identify the state?}
The representation argument in Appendix~\ref{app:representation-bounds} relates the approximation loss of a history statistic $Z_t$ to its Bayes state-decoding error $e_t(Z)$. Thus the key question is whether the lagged context or the probe fingerprint contains enough information to recover the current hidden state. Let
\(
    \gamma_a := \min_{s \neq s'} \lvert \mu_{s,a} - \mu_{s',a} \rvert
\)
denote the single-arm separation gap for arm $a$, and let
\(
    \beta_1 := \sup_s\!\left(1-\max_u P_{su}\right)
\)
denote the one-step predictability error. The lagged context $(a_{t-1}, r_{t-1})$ is informative for $s_t$ when $\gamma_{a_{t-1}}$ is large relative to $\sigma$ and $\beta_1$ is small: the first condition helps decode $s_{t-1}$, and the second propagates that estimate to $s_t$. A fingerprint over an arm set $B\subseteq\mathcal A$ replaces the single-arm gap by the joint separation
\[
    \Gamma_B := \min_{s \neq s'} \big\lVert (\mu_{s,a})_{a\in B} - (\mu_{s',a})_{a\in B} \big\rVert_2,
\]
with $B=\{0,1\}$ in the two-arm experiments. This can separate states even when individual arms are ambiguous, as in \Cref{tbl:state-ambiguity}. 

Sequential fingerprinting additionally pays a state-mismatch cost because the hidden state can transition between the two consecutive arms used to build the fingerprint. For a two-arm sequential probe, this mismatch event has probability at most $q$ (the one-step state-change probability of \Cref{thm:tau}), which is the structural reason that sequential variants degrade under fast-switching chains while randomized variants do not. A detailed analysis can be found in \Cref{app:representation-bounds}.

\begin{algorithm}[htbp]
\caption{Sequential Probing Upper Confidence Bound (SP-UCB)}
\label{alg:sp-ucb}
\KwData{$\alpha, \lambda, \tau, T, K \gets 2$}
\tcp{(Omitted) Lines 1-4 from \Cref{alg:rp-ucb}}
\For{$t \gets 1, 2, \ldots, T$}{
    \tcp{(Omitted) Lines 6-8 from \Cref{alg:rp-ucb}}
    \eIf{$t \bmod \tau \in \{0, 1\}$}{
        $\mathtt{mode} \gets \text{probe}$\tcp*{Probe mode}
        $a_t \gets t \bmod \tau$\;
    }{
        $\mathtt{mode} \gets \text{exploit}$\tcp*{Exploit mode}
        \For{$a \gets 0, 1$}{
            Compute $\mathrm{UCB}_{t,a}$ using LinUCB model $\mathcal{M}_a$ with context $\boldsymbol{\phi}_t$\;
        }
        $a_t \gets \argmax_{a\in\{0,1\}} \mathrm{UCB}_{t, a}$\;
    }
    Play action $a_t$ and observe reward $r_t$\;
    \If{$\mathtt{mode} = \text{probe}$ and $a_t = 1$}{
        $(r_0^{\mathrm{fp}}, r_1^{\mathrm{fp}}) \gets (r_{t-1}, r_t)$\;
    }
    Update $\mathcal{M}_{a_t}$ with $(\boldsymbol{\phi}_t, r_t)$\;
}
\end{algorithm}

\section{Upper Confidence Bound with Adaptive Probing}
Fixed-schedule probing in RP-UCB and SP-UCB collects information at predetermined times, but each probe can incur regret. The same interval is used early and late in learning, even though model uncertainty changes over time. Early in learning, frequent probing can accelerate convergence; later, when the model is more certain, excessive probing wastes exploitation rounds. If prediction errors spike, indicating state drift or model mismatch, probing may be useful outside the fixed schedule.

Adaptive probing makes the probe decision data dependent rather than periodic. The adaptive probing variant of RP-UCB (\Cref{alg:rp-ucb}) is AdaRP-UCB, presented in \Cref{alg:adarp-ucb}; the analogous variant of SP-UCB (\Cref{alg:sp-ucb}) is AdaSP-UCB, shown in \Cref{alg:adasp-ucb} in the appendix.

\subsection{Adaptive Gating Mechanism}
\label{sec:adaptive-gating}
We use three probe triggers, each based on a different signal. Probing is triggered when any gate activates, subject to a minimum gap constraint.

\begin{enumerate}
    \item Residual gate detects model drift. The residual gate monitors prediction errors. If the model accurately captures the state-reward relationship, residuals should be small. Large residuals indicate either model misspecification or a state transition that invalidates the current fingerprint. At each round, we compute the Studentized residual
    \[
        z_{t-1} = \frac{r_{t-1} - \hat{r}_{t-1}}{\sqrt{\mathrm{Var}(\hat{r}_{t-1}) + \sigma_0^2}}
    \]
    where $\hat{r}_{t-1} = \boldsymbol{\phi}_{t-1}^\top \hat{\boldsymbol{\theta}}_{a_{t-1}}$ is the predicted reward, $\mathrm{Var}(\hat{r}_{t-1}) = \boldsymbol{\phi}_{t-1}^\top \mathbf{A}_{a_{t-1}}^{-1} \boldsymbol{\phi}_{t-1}$ is the model uncertainty, and $\sigma_0^2$ is ideally the observation noise variance $\sigma^2$, but serves primarily as a floor to prevent the residual gate from triggering too often.
    
    The residual gate activates when the absolute residual exceeds a threshold.
    \[
        g_{\mathrm{res}} \coloneq \mathbf{1}\left[ |z_{t-1}| \geq z_{\mathrm{thresh}} \right].
    \]
    Under correct model specification and no state transition, $z_{t-1}$ is approximately standard normal. Setting $z_{\mathrm{thresh}} \approx 2$ corresponds to probing when residuals fall outside a 95\% confidence interval.
    
    \item Uncertainty gate handles small UCB margins. The gate monitors the gap between the largest two UCB scores. When the gap is small, an additional probe can update the fingerprint before the algorithm commits to an arm. We compute the UCB margin as $m_t = \left| \mathrm{UCB}_{t,0} - \mathrm{UCB}_{t,1} \right|$.
    
    The uncertainty gate activates when the margin falls below a threshold.
    \[
        g_{\mathrm{unc}} \coloneq \mathbf{1}\left[ m_t \leq m_{\mathrm{thresh}} \right].
    \]
    A small margin means the confidence intervals overlap substantially, so the algorithm refreshes the fingerprint before choosing between the top arms.
    
    \item Staleness gate bounds fingerprint age. The staleness gate bounds the age of the current fingerprint, covering cases where residuals and margins do not trigger a refresh. We use an exponential hazard function based on time since last probe
    \[
        h(t) = 1 - \exp\left( -\lambda_h (t - t_{\mathrm{probe}}) \right).
    \]
    
    The staleness gate activates when the hazard exceeds a threshold.
    \[
        g_{\mathrm{haz}} \coloneq \mathbf{1}\left[ h(t) \geq \delta_h \right].
    \]
    Equivalently, the threshold defines an age cap $H=\lceil-\log(1-\delta_h)/\lambda_h\rceil$ for the current fingerprint. With $\lambda_h=0.1$ and $\delta_h=0.5$, $H\approx 7$. Residual and uncertainty gates can refresh earlier when prediction errors or small margins indicate that the current fingerprint is unreliable.
\end{enumerate}

Combining the three gates, the algorithm probes when any gate activates, subject to a minimum gap.
\[
    \text{Probe at } t \iff \left( g_{\mathrm{res}} \lor g_{\mathrm{unc}} \lor g_{\mathrm{haz}} \right) \land \left( t - t_{\mathrm{probe}} \geq \tau_{\mathrm{min}} \right).
\]
The minimum gap $\tau_{\min}$ prevents oscillatory behavior where the algorithm probes excessively in response to transient fluctuations.

Unless stated otherwise, experiments use $z_{\mathrm{thresh}}=2.0$, $m_{\mathrm{thresh}}=0.05$, $\lambda_h=0.1$, $\delta_h=0.5$, and $\tau_{\min}=2$. The residual floor $\sigma_0$ is set to the observation-noise scale $\sigma$, which is $0.05$ in the robustness experiments. Gates are evaluated before action selection using current LinUCB scores; for $K>2$, the uncertainty margin is the top-two UCB gap. A probe updates the timestamp at its first round, all probe rounds update the learner, and ties follow the fixed arm ordering.

\begin{algorithm}[htbp]
\caption{Adaptive Randomized Probing Upper Confidence Bound (AdaRP-UCB)}
\label{alg:adarp-ucb}
\SetKwProg{Fn}{Function}{:}{}
\SetKwFunction{FComputeGates}{ComputeGates}
\KwData{$\alpha$, $\lambda$, $z_{\mathrm{thresh}}$, $\sigma_0$, $m_{\mathrm{thresh}}$, $\lambda_h$, $\delta_h$, $\tau_{\mathrm{min}}$, $T$}
\For{$a \gets 0, 1$}{
    Initialize LinUCB model $\mathcal{M}_a$ with exploration parameter $\alpha$ and regularization $\lambda$
}
$(r_0^{\mathrm{fp}}, r_1^{\mathrm{fp}}) \gets (0, 0)$\tcp*{Initialize fingerprint}
$(a_0, r_0) \gets (0, 0)$ \tcp*{Init. prev. action and reward}
$t_\mathrm{probe} \gets 0$\;
\For{$t = 1, 2, \ldots, T$}{
    $\boldsymbol\phi_{\mathrm{fp}} \gets (r_0^{\mathrm{fp}}, r_1^{\mathrm{fp}})$\tcp*{Fingerprint features}
    $\boldsymbol\phi_{\mathrm{lag}} \gets (a_{t-1}, r_{t-1})$\tcp*{Lagged features}
    $\boldsymbol{\phi}_t \gets \boldsymbol\phi_{\mathrm{fp}} \oplus \boldsymbol\phi_{\mathrm{lag}}$\;
    \For{$a \gets 0, 1$}{
        Compute $\mathrm{UCB}_{t,a}$ using LinUCB model $\mathcal{M}_a$ with context $\boldsymbol{\phi}_t$\;
    }
    
    $g \gets \FComputeGates$\;
    \eIf{$g$}{
        $\mathtt{mode} \gets \text{probe}$\tcp*{Probe mode}
        $(a_t^C, a_t^T) \gets (0, 1)$\;
        $t_{\mathrm{probe}} \gets t$\;
    }{
        $\mathtt{mode} \gets \text{exploit}$\tcp*{Exploit mode}
        $a^* \gets \arg\max_a \mathrm{UCB}_{t,a}$\;
        $(a_t^C, a_t^T) \gets (a^*, a^*)$\;
    }
    Play actions $(a_t^C, a_t^T)$ and observe reward $(r_t^C, r_t^T)$\;
    Update LinUCB models $\mathcal{M}_{a_t^C}$ and $\mathcal{M}_{a_t^T}$ with $(\boldsymbol{\phi}_t, r_t^C)$ and $(\boldsymbol{\phi}_t, r_t^T)$\;
    \eIf{$\mathtt{mode} = \text{probe}$}{
        $(r_0^{\mathrm{fp}}, r_1^{\mathrm{fp}}) \gets (r_t^C, r_t^T)$\;
        \eIf{$r_t^T > r_t^C$}{
            $(a_t, r_t) \gets (a_t^T, r_t^T)$\;
        }{
            $(a_t, r_t) \gets (a_t^C, r_t^C)$\;
        }
    }{
        $(a_t, r_t) \gets (a_t^C, (r_t^C + r_t^T) / 2)$\;
    }
}
\Fn{ComputeGates}{
    \tcp{Residual gate}
    $\hat{r}_{t-1} \gets \boldsymbol{\phi}_{t-1}^\top \hat{\boldsymbol{\theta}}_{a_{t-1}}$\;
    $v_{t-1} \gets \boldsymbol{\phi}_{t-1}^\top \mathbf{A}_{a_{t-1}}^{-1} \boldsymbol{\phi}_{t-1}$\;
    $z_{t-1} \gets \frac{r_{t-1} - \hat{r}_{t-1}}{\sqrt{v_{t-1} + \sigma_0^2}}$\;
    $g_{\mathrm{res}} = \mathbf{1}[|z_{t-1}| \geq z_{\mathrm{thresh}}]$\;
    
    \tcp{Uncertainty gate}
    $m_t = |\mathrm{UCB}_{t,0} - \mathrm{UCB}_{t,1}|$\;
    $g_{\mathrm{unc}} = \mathbf{1}[m_t \leq m_{\mathrm{thresh}}]$\;

    \tcp{Staleness gate}
    $g_{\mathrm{haz}} = \mathbf{1}[1 - \exp(-\lambda_h (t - t_{\mathrm{probe}})) \geq \delta_h]$\;

    \Return $(g_{\mathrm{res}} \vee g_{\mathrm{unc}} \vee g_{\mathrm{haz}}) \wedge (t - t_{\mathrm{probe}} \geq \tau_{\min})$\;
}
\end{algorithm}

\begin{algorithm}[htbp]
\caption{Adaptive Sequential Probing Upper Confidence Bound (AdaSP-UCB)}
\label{alg:adasp-ucb}
\SetKwFunction{FComputeGates}{ComputeGates}
\KwData{$\alpha$, $\lambda$, $z_{\mathrm{thresh}}$, $\sigma_0$, $m_{\mathrm{thresh}}$, $\lambda_h$, $\delta_h$, $\tau_{\mathrm{min}}$, $T$}
\tcp{(Omitted) Lines 1-5 from \Cref{alg:adarp-ucb}}
$\mathtt{mode} \gets \text{exploit}$\;
\For{$t = 1, 2, \ldots, T$}{
    \tcp{(Omitted) Lines 7-12 from \Cref{alg:adarp-ucb}}
    \If{$\mathtt{mode} = \text{exploit}$}{
        \If{$g$}{
            $\mathtt{mode} \gets \text{probe}$\tcp*{Probe mode}
            $\mathtt{probe\_arm} \gets 0$\tcp*{Start probing arm $0$}
        }
    }
    \eIf{$\mathtt{mode} = \text{probe}$}{
        $a_t \gets \mathtt{probe\_arm}$\;
        $\mathtt{probe\_arm} \gets \mathtt{probe\_arm} + 1$\tcp*{Probe the next arm next time}
    }{
        $a_t \gets \arg\max_a \mathrm{UCB}_{t,a}$\;
    }
    Play action $a_t$ and observe reward $r_t$\;
    \If{$\mathtt{mode} = \text{probe}$ and $\mathtt{probe\_arm} = 2$}{
        $\mathtt{mode} \gets \text{exploit}$\tcp*{Completed sequential probing}
        $(r_0^{\mathrm{fp}}, r_1^{\mathrm{fp}}) \gets (r_{t-1}, r_t)$\;
    }
    Update $\mathcal{M}_{a_t}$ with $(\boldsymbol{\phi}_t, r_t)$\;
}
\end{algorithm}

\subsection{A Unified View}
\label{sec:unified-view}

The algorithms differ along two axes: the context supplied to LinUCB and the schedule used to collect probe information. LC-UCB uses only the lagged pair $(a_{t-1}, r_{t-1})$. RP-UCB and SP-UCB append a probe fingerprint, with periodic probes under the fixed-schedule variants and residual, margin, and hazard triggers under the adaptive variants. Randomized probing requires concurrent units but observes same-round fingerprints; sequential probing uses one unit but requires the state to remain stable during the probe window. We use LinUCB because the gates need predicted rewards, uncertainty estimates, and top-two UCB margins from the same feature representation.

\section{Experiments}
We evaluate the adaptive probing algorithms against classical bandits (UCB1, Thompson Sampling, EXP3) and non-stationary bandits (Sliding Window UCB, Discounted UCB, EXP3-S).\footnote{Please refer to \Cref{tbl:algorithm_reference} for quick reference of all algorithms that we consider.} The environment is a two-armed bandit with $S$ hidden states, Gaussian rewards $\mathcal{N}(\mu_{s,a}, \sigma^2)$, and Markov transitions with self-transition probability $p_{\mathrm{stay}}$. The randomized algorithms (AdaRP-UCB and RP-UCB) control two contemporaneous units per decision time: probe rounds assign the two units to arms $0$ and $1$, while exploit rounds assign both units to the selected arm. We report their cumulative regret on a per-unit scale, so a two-unit round contributes the average of the two units' instantaneous gaps before summing over the same $T$ decision times used by the single-unit methods. For these two-unit methods, each observed unit reward uses variance $2\sigma^2$, making same-round fingerprints noisier than in the single-unit reward model.

We vary $S \in \{2, 10, 20, 50\}$, $p_{\mathrm{stay}} \in \{0.5, 0.8, 0.9, 0.95, 0.99\}$, $\sigma \in \{0.01, 0.05, 0.1, 0.5\}$, and $T \in \{500, 1000, 5000,\allowbreak 20000\}$. For each configuration, we sample $128$ reward matrices with $\mu_{s,a} \sim \mathrm{Uniform}[0,1]$ and run each algorithm $5$ times. We report cumulative regret.

\begin{table*}[t]
\centering
\footnotesize
\setlength{\tabcolsep}{2pt}
\renewcommand{\arraystretch}{1.18}
\begin{tabular}{@{}rlrrrrrrrrrrrrr@{}}
\toprule
\textbf{Value} & \textbf{Perturbation}
& \multicolumn{2}{c}{\textbf{Two-unit randomized}}
& \multicolumn{10}{c}{\textbf{Single-unit}}
& \multicolumn{1}{c}{\textbf{Oracle}} \\
\cmidrule(lr){3-4}\cmidrule(lr){5-14}\cmidrule(l){15-15}
& & \multicolumn{1}{c}{\shortstack{\textbf{AdaRP}\\\textbf{UCB}}}
& \multicolumn{1}{c}{\shortstack{\textbf{RP}\\\textbf{UCB}}}
& \multicolumn{1}{c}{\shortstack{\textbf{AdaSP}\\\textbf{UCB}}}
& \multicolumn{1}{c}{\textbf{D-UCB}}
& \multicolumn{1}{c}{\textbf{EXP3}}
& \multicolumn{1}{c}{\textbf{EXP3-S}}
& \multicolumn{1}{c}{\textbf{LC-TS}}
& \multicolumn{1}{c}{\textbf{LC-UCB}}
& \multicolumn{1}{c}{\textbf{SP-UCB}}
& \multicolumn{1}{c}{\textbf{SW-UCB}}
& \multicolumn{1}{c}{\textbf{TS}}
& \multicolumn{1}{c}{\textbf{UCB1}}
& \multicolumn{1}{c}{\shortstack{\textbf{Opt.}\\\textbf{single arm}}} \\
\midrule
 & Default & \best{452.22} & 943.64 & \best{707.69} & 1553.94 & 4516.35 & 3452.92 & 1299.25 & 1530.02 & 1137.14 & 1476.74 & 5037.26 & 2763.03 & 4451.31 \\
\addlinespace[2pt]
2  & Number of States & \best{342.58} & 933.81 & 585.72 & 1400.98 & 2366.72 & 2971.22 & 153.48 & \best{118.35} & 327.74 & 1393.97 & 2586.57 & 1102.10 & 2074.36 \\
20 & Number of States & \best{467.12} & 948.49 & \best{712.65} & 1564.25 & 5149.11 & 3536.75 & 1450.52 & 1914.41 & 1244.48 & 1484.23 & 5681.02 & 3167.98 & 5192.14 \\
50 & Number of States & \best{468.95} & 966.04 & \best{709.16} & 1585.24 & 5737.84 & 3623.26 & 1552.16 & 2066.59 & 1292.20 & 1500.65 & 6252.61 & 3504.34 & 5840.72 \\
\addlinespace[2pt]
0.50 & Self-transition Prob. & \best{4007.60} & 4630.15 & 4516.70 & 4377.77 & 4751.64 & 5509.11 & 3659.89 & \best{3624.84} & 3739.32 & 4952.51 & 5323.92 & 4377.71 & 4572.14 \\
0.80 & Self-transition Prob. & \best{2368.62} & 3412.53 & 2964.17 & 2995.52 & 4637.48 & 5326.39 & \best{2351.42} & 2407.24 & 2554.18 & 4039.61 & 5237.72 & 3487.36 & 4445.60 \\
0.90 & Self-transition Prob. & \best{1600.62} & 2510.65 & 2153.95 & 2567.11 & 4706.50 & 5198.32 & \best{2031.90} & 2123.06 & 2259.17 & 3441.85 & 5395.82 & 3197.33 & 4541.32 \\
0.95 & Self-transition Prob. & \best{1065.21} & 2017.19 & \best{1514.30} & 2209.35 & 4885.10 & 4977.90 & 1694.93 & 1783.73 & 1907.38 & 2807.29 & 5490.35 & 3064.73 & 4760.40 \\
\addlinespace[2pt]
0.05 & Reward Noise SD. & \best{644.76}  & 1075.98 & \best{814.56} & 1548.92 & 4589.35 & 3502.60 & 1225.25 & 1363.40 & 1127.96 & 1491.20 & 5287.53 & 2853.86 & 4573.26 \\
0.10 & Reward Noise SD. & \best{1330.41} & 1379.29 & 1444.04 & 1541.61 & 4558.92 & 3458.40 & \best{1157.71} & 1254.24 & 1202.42 & 1507.11 & 5267.70 & 2783.66 & 4491.77 \\
0.50 & Reward Noise SD. & 3960.63 & \best{3872.88} & 4118.41 & \best{1870.63} & 4637.03 & 3573.40 & 2842.84 & 2480.63 & 2840.06 & 2041.06 & 5741.33 & 3119.61 & 4544.59 \\
\addlinespace[2pt]
500  & Number of Rounds & \best{16.01}  & 24.69  & 23.85  & 37.32  & 98.90   & 90.01  & 36.37  & \best{23.20} & 23.38  & 36.78  & 95.03   & 44.90  & 117.45 \\
1000 & Number of Rounds & \best{30.22}  & 48.74  & \best{45.40}  & 72.79  & 207.13  & 178.43 & 70.23  & 53.38  & 48.18  & 72.54  & 215.67  & 99.92  & 233.18 \\
5000 & Number of Rounds & \best{125.89} & 242.22 & \best{192.60} & 385.43 & 1131.87 & 872.31 & 341.68 & 364.29 & 285.26 & 371.11 & 1286.11 & 612.66 & 1135.05 \\
\bottomrule
\end{tabular}
\vspace{2pt}
\caption{Cumulative regret across environment perturbations. Best values within each deployment group are highlighted.}
\label{tbl:performance_metrics}

\end{table*}

The mean cumulative regrets are in \Cref{tbl:performance_metrics}. Among randomized methods, AdaRP-UCB is best in $12$ of $13$ configurations and improves most over RP-UCB in volatile environments. Among single-unit methods, LC-UCB, LC-TS, AdaSP-UCB, and SP-UCB consistently have lower regret than the classical and non-stationary baselines. AdaSP-UCB has the lowest regret under low reward noise, slow mixing, or larger state spaces, while lagged-context methods are stronger under faster mixing. The main exception is high noise ($\sigma=0.5$, with expected arm gap about $0.1$), where fingerprints are weak and D-UCB or SW-UCB can do better. The winning-rate analysis in \Cref{fig:h2h} gives a complementary view of these comparisons across parameter regimes.

\begin{figure}[t]
    \centering
    \includegraphics[width=0.9\linewidth]{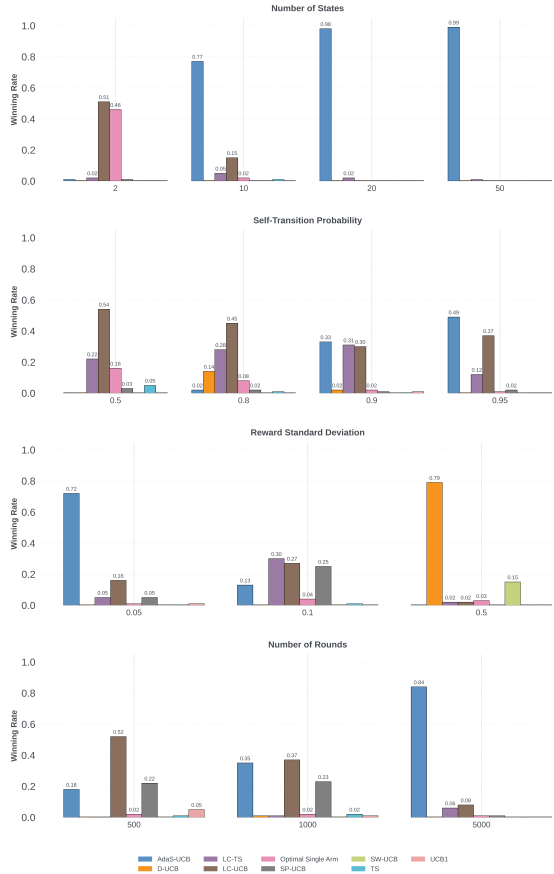}
    \Description[Head-to-head winning rates under different problem-specific parameters.]{
    Head-to-head winning rates across sweeps over the number of states, self-transition probability, reward noise, and number of rounds.
    }
    \caption{Head-to-head winning rates under different problem-specific parameters.}
    \label{fig:h2h}
\end{figure}

We also measure optimal-arm frequency, the fraction of rounds in which an algorithm selects the oracle best arm $a_t^*$ for the realized hidden state sequence. \Cref{fig:all_freq} reports this quantity for the default setting and three single-axis perturbations, with curves smoothed over a sliding window.

\begin{figure}[t]
    \centering
    \begin{subfigure}[b]{0.48\linewidth}
        \centering
        \includegraphics[width=\linewidth]{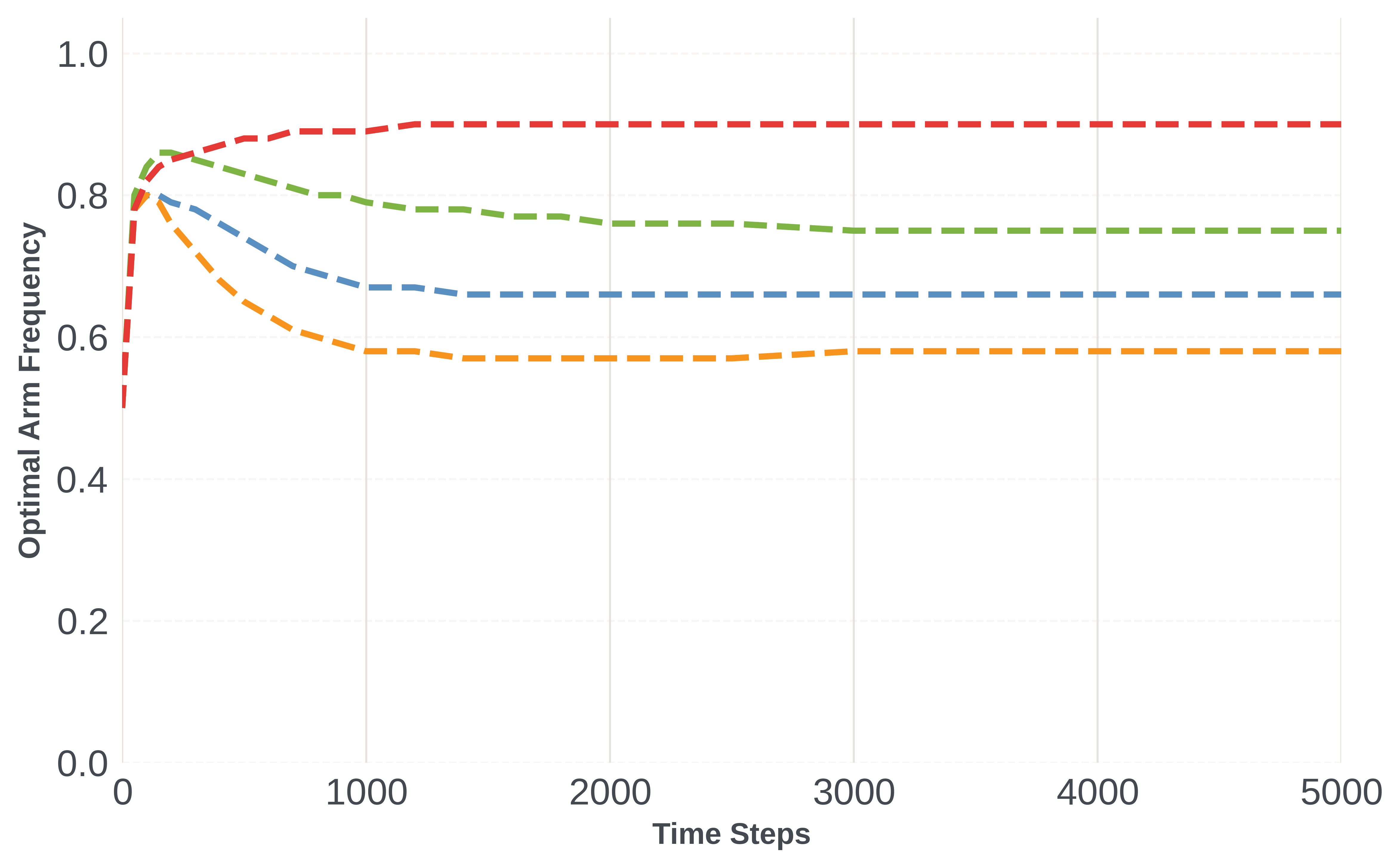}
        \caption{Default setting.}
        \label{fig:freq1}
    \end{subfigure}
    \hfill
    \begin{subfigure}[b]{0.48\linewidth}
        \centering
        \includegraphics[width=\linewidth]{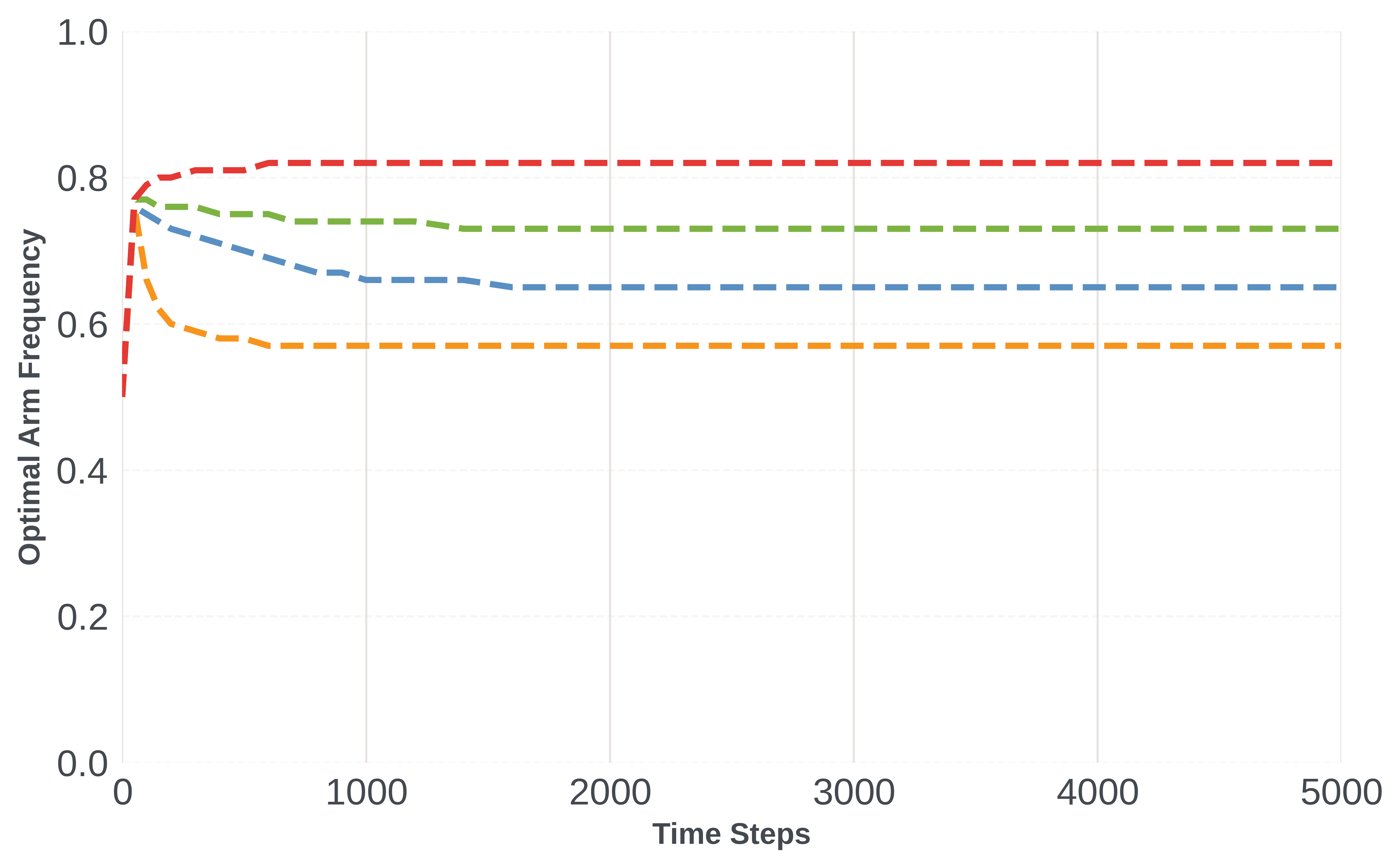}
        \caption{$p_{\mathrm{stay}}=0.5$.}
        \label{fig:freq2}
    \end{subfigure}

    \begin{subfigure}[b]{0.48\linewidth}
        \centering
        \includegraphics[width=\linewidth]{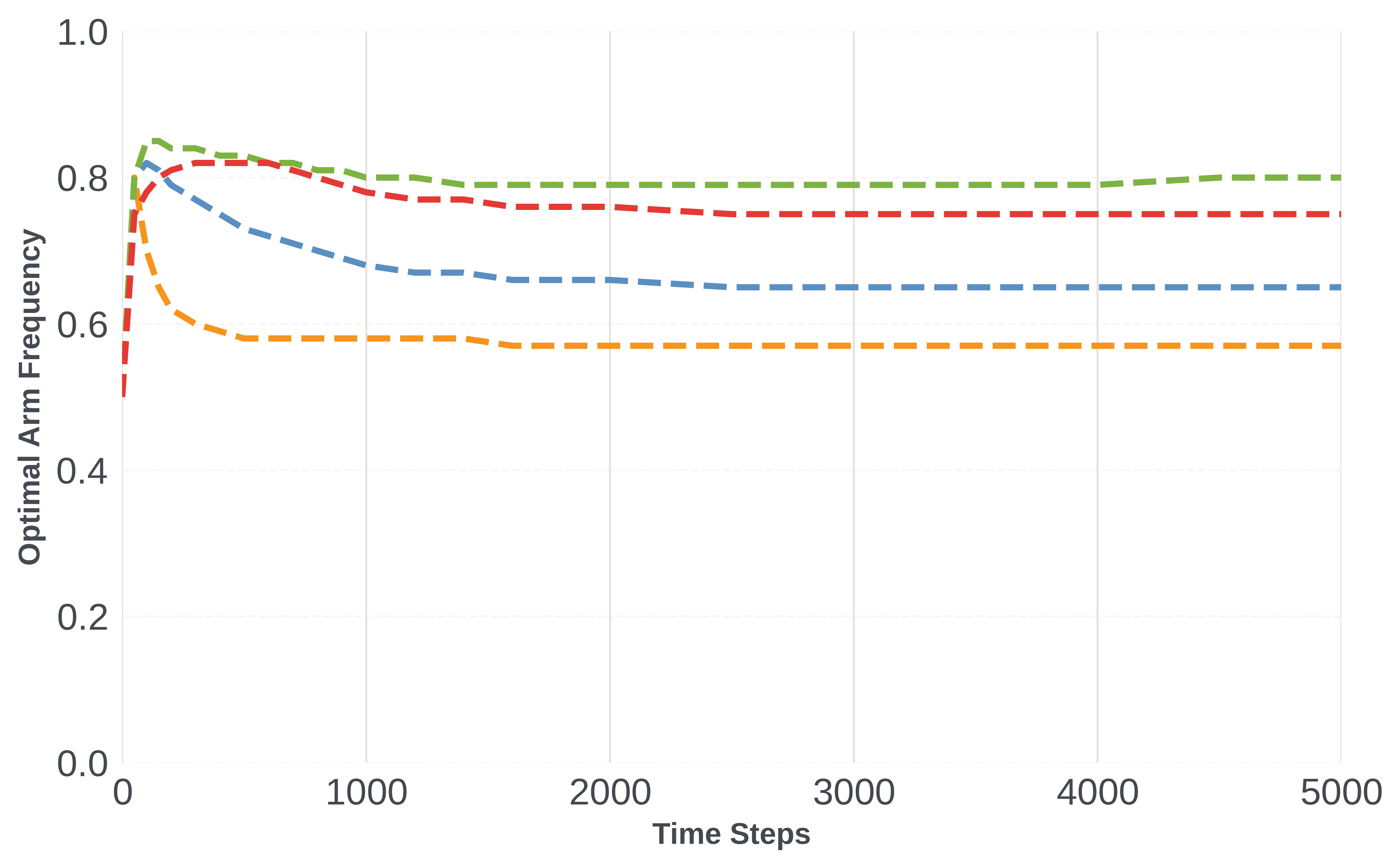}
        \caption{Reward noise, $\sigma=0.1$.}
        \label{fig:freq3}
    \end{subfigure}
    \hfill
    \begin{subfigure}[b]{0.48\linewidth}
        \centering
        \includegraphics[width=\linewidth]{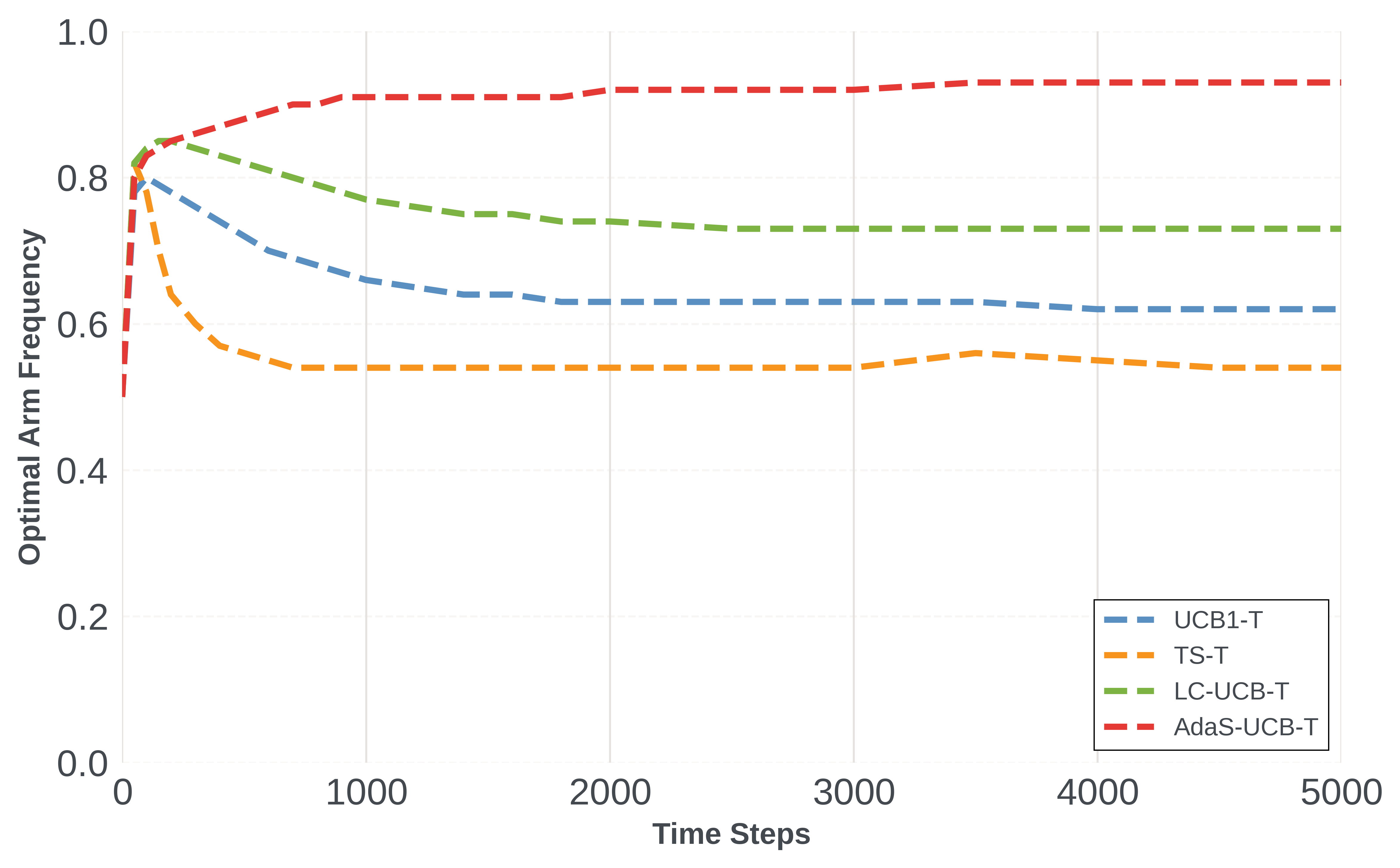}
        \caption{Number of states, $S=50$.}
        \label{fig:freq4}
    \end{subfigure}
    \Description[Frequency of pulling the optimal arm across time.]{
    Frequency with which each algorithm selects the oracle optimal arm across time under the default setting and three perturbed environments.
    }
    \caption{Frequency of pulling the optimal arm across time.}
    \label{fig:all_freq}
\end{figure}

\subsection{Hyperparameter Sensitivity}
\label{sec:sensitivity}

We sweep each adaptive-probing threshold across the default environment and three perturbations: $p_{\mathrm{stay}} = 0.5$, $\sigma = 0.1$, and $S = 50$. The uncertainty margin is stable: $m_{\mathrm{thresh}}=0.05$ is within $5\%$ of the per-environment best in all eight environment-algorithm pairs. The residual gate prefers $z_{\mathrm{thresh}}=2.5$ in all eight pairs, though the default $2.0$ is within $10\%$ in seven. Probe frequency matters more: $H=\lceil-\log(1-\delta_h)/\lambda_h\rceil=12$ is best in all eight pairs, and $\tau_{\min}=4$ in seven. Thus the defaults $(2.0,0.05,7,2)$ are conservative; raising $z_{\mathrm{thresh}}$, $H$, or $\tau_{\min}$ reduces excess probing, while lowering $H$ or $\tau_{\min}$ combats stale fingerprints.

\subsection{Component Ablation}
\label{sec:ablation}

We ablate the LinUCB features and schedule by comparing the full model, no-$\phi_{\mathrm{lag}}$, no-$\phi_{\mathrm{fp}}$, and fixed-schedule variants. In the sticky-state example ($p_{\mathrm{stay}}=0.95$), removing $\phi_{\mathrm{fp}}$ raises AdaSP regret from about $238$ to $374$, close to the fixed-schedule variant ($377$), so the fingerprint is the main gain. For randomized probing in the same regime, full AdaRP has regret $103$ versus $162$ for fixed-schedule RP-UCB, while dropping either context component gives about $108$; here adaptive timing is the main gain. In the fast-switching regime ($p_{\mathrm{stay}}=0.5$), AdaSP, no-$\phi_{\mathrm{fp}}$, and fixed-schedule AdaSP are nearly identical because roughly half of the sequential probe windows span a state transition.

\subsection{Robustness to Context Misspecification}
\label{sec:misspec}

To test for robustness, we compare a \emph{matched} learner that receives the correct sufficient statistic with a \emph{misspecified} learner that uses \eqref{eq:combined-features}; the gates and exploration parameters are unchanged, and $\sigma=0.05$. In a ten-state Markov chain, matched AdaRP has regret $111$ versus $189$ for misspecified AdaRP at $p_{\mathrm{stay}}=0.95$, and $91$ versus $580$ at $p_{\mathrm{stay}}=0.5$; misspecified AdaRP still beats SW-UCB ($392$, $635$) and D-UCB ($468$, $617$). Sequential probing fails faster under switching: at $p_{\mathrm{stay}}=0.5$, matched AdaSP has regret $189$, while misspecified AdaSP rises to $738$.

Two additional checks isolate memory mismatch and joint separation. When rewards depend on both previous and current states, misspecified AdaRP remains between matched AdaRP and non-stationary baselines: at $\rho=0.25$, the regrets are $92$, $145$, $298$, and $418$ for matched AdaRP, misspecified AdaRP, SW-UCB, and D-UCB; at $\rho=0.5$, they are $85$, $133$, $283$, and $400$. In an overlap sweep with small joint-separation gap $\delta=0.2$, the same ordering gives $67$, $128$, $384$, and $394$. Together, the sensitivity, ablation, overlap, and many-arm checks are consistent with the theory: thresholds control the probe-cost/staleness trade-off, feature ablations separate representation value from scheduling value, overlap controls fingerprint separation, and fast switching or larger sequential probes increase the mismatch term. AdaRP has low regret when fingerprints separate states and are refreshed often enough, while sequential probing also requires a low probability of state changes within a probe window.

\section{Conclusion}
This paper studies bandits in which an unobserved Markov state changes the reward ranking of the arms. Standard bandit baselines incur persistent dynamic regret in this setting because their estimates average over latent states. We propose algorithms that pass LinUCB a summary built from recent action-reward history and, when possible, probe rewards from multiple arms. The analysis separates three sources of regret: representation ambiguity, probe cost, and drift after a probe becomes stale. The experiments show gains when the summaries separate the states and remain current, and identify the main failure modes: high noise, weak fingerprint separation, and state changes during sequential probes. The intended regime is therefore one where the latent state is predictable over short horizons or refreshed fingerprints remain current; richer history summaries and finite-sample analysis for the full LinUCB instantiations are left for future work.


\bibliographystyle{ACM-Reference-Format}
\balance
\bibliography{reference}

\appendix

\section{Summary of Related Work}
We summarize the similarities and differences between our setting (in the ``Latent-State'' column) and existing settings in the literature in \Cref{tbl:comparison}.

\begin{table*}[!htbp]
\centering
\begin{adjustbox}{max width=\textwidth}
\begin{tabular}{lcccccccc}
\toprule
\textbf{Feature} & \textbf{Classical} & \textbf{Contextual} & \textbf{Non-Stat.} & \textbf{Adversarial} & \textbf{Restless} & \textbf{RL / POMDP} & \textbf{LMDP} & \textbf{Latent-State} \\
\midrule
\textbf{Stationarity} & Yes & Yes & Arbitrary drift & Worst-case & Markovian & Markovian & Episodic & Markovian \\
\textbf{State} & None & Observable & N/A & None & Observable & Hidden/Partial & Hidden & Hidden \\
\textbf{State Observability} & N/A & Full & N/A & N/A & Observable & Partial & Hidden & Unobservable \\
\textbf{Knowledge of Trans. Prob.} & N/A & N/A & N/A & N/A & Known & Unknown/Learned & Unknown & Not modeled \\
\textbf{Action Affects Transition} & No & No & No & No & Yes & Yes & No & No \\
\textbf{Trajectories Needed} & Single & Single & Single & Single & Single & Multiple & Multiple & Single \\
\textbf{Regret Type} & Static & Contextual & Dynamic & Fixed & Index-based & Cumulative & Episodic & Dynamic \\
\bottomrule
\end{tabular}
\end{adjustbox}
\caption{Comparison of related work}
\label{tbl:comparison}
\end{table*}

\section{Summary of Algorithms}
We present the detailed description of the algorithms compared in \Cref{tbl:algorithm_reference}.

\begin{table*}[!htbp]
\centering
\small
\begin{tabular}{p{\dimexpr 0.2\linewidth-2\tabcolsep}p{\dimexpr 0.15\linewidth-2\tabcolsep}p{\dimexpr 0.65\linewidth-2\tabcolsep}}
\toprule
\textbf{Algorithm} & \textbf{Type} & \textbf{Key Mechanism} \\
\midrule
\multicolumn{3}{l}{\textbf{\textit{Randomized Algorithms}}} \\
\midrule
\textbf{AdaRP-UCB} & Adaptive & LinUCB on synchronized probe fingerprints plus lagged context; residual, uncertainty, and hazard gates trigger probes. Assumes simultaneous units and benefits from fresh, separated fingerprints. \\
\addlinespace[0.2ex]
\textbf{RP-UCB} & Probing & Fixed-schedule synchronized probes every $\tau$ rounds; exploits with LinUCB on joint fingerprints and lagged context. Pays probe cost but no within-probe drift. \\
\midrule
\multicolumn{3}{l}{\textbf{\textit{Single-Unit Algorithms}}} \\
\midrule
\textbf{AdaSP-UCB} & Adaptive & Single-unit analog of AdaRP-UCB using consecutive-arm fingerprints and the same gates. Does not need simultaneous units, but pays sequential mismatch and staleness cost. \\
\addlinespace[0.2ex]
\textbf{LC-UCB} & Contextual & LinUCB using previous action and reward features. Works when a recent single-arm observation is informative about the current state. \\
\addlinespace[0.2ex]
\textbf{LC-TS} & Contextual & Thompson sampling with the same lagged action-reward features as LC-UCB. \\
\addlinespace[0.2ex]
\textbf{SP-UCB} & Probing & Fixed two-round sequential probing every $\tau$ steps, followed by LinUCB on synthetic fingerprints and lagged context. Sensitive to state changes during the probe. \\
\addlinespace[0.2ex]
\textbf{D-UCB} & Non-stationary & Discounted UCB with exponentially decaying arm statistics. \\
\addlinespace[0.2ex]
\textbf{SW-UCB} & Non-stationary & UCB over a fixed window of recent arm observations. \\
\addlinespace[0.2ex]
\textbf{EXP3-S} & Non-stationary & EXP3 variant with additive smoothing for changing environments. \\
\addlinespace[0.2ex]
\textbf{UCB1} & Classical & Standard UCB with global arm statistics and no latent-state adaptation. \\
\addlinespace[0.2ex]
\textbf{TS} & Classical & Thompson sampling with independent arm posteriors. \\
\addlinespace[0.2ex]
\textbf{EXP3} & Adversarial & Exponential weights with importance-weighted rewards and uniform exploration. \\
\addlinespace[0.2ex]
\midrule
\multicolumn{3}{l}{\textbf{\textit{Oracle Algorithms}}} \\
\midrule
\textbf{Optimal Single Arm} & Oracle & Best fixed arm under the stationary state distribution. \\
\bottomrule
\end{tabular}
\caption{Summary of Algorithms}
\label{tbl:algorithm_reference}
\end{table*}

\section{Proof of \texorpdfstring{\Cref{thm:tau}}{Theorem 4.2}}
\subsection{Representation bounds for hidden-state summaries}
\label{app:representation-bounds}

This subsection proves the representation and state-separation claims in the discussion following \Cref{thm:tau}. These results
justify when lagged contexts and probe fingerprints give small probe-time
identification error. The probe-scheduling bounds themselves are proved in
\Cref{app:probe-schedule-proof}.

\paragraph{Representation loss.}
For any history statistic $Z_t$, let $\pi_t^Z(Z_t)$ be the best arm rule measurable with respect to $Z_t$. Adding and subtracting $\mu_{s_t,\pi_t^Z(Z_t)}$ gives the exact decomposition
\[
R_T(A)=\mathrm{App}_T(Z)+\mathrm{Reg}_T^Z(A).
\]
Let $e_t(Z)=\inf_d\mathbb P(d(Z_t)\ne s_t)$ be the Bayes state-decoding error. For any decoder $d$, the decoded-state rule $a^\ast_{d(Z_t)}$ is $Z_t$-measurable. Since $\pi_t^Z$ is the best $Z_t$-measurable rule in expected reward, its expected loss is no larger than that of this decoded-state rule. Thus
\[
\begin{aligned}
&\mathbb E\!\left[\mu_{s_t,a^\ast_{s_t}}-\mu_{s_t,\pi_t^Z(Z_t)}\right] \\
&\quad \le \Delta_{\max}\mathbb P(d(Z_t)\ne s_t).
\end{aligned}
\]
Taking the infimum over $d$ and summing over $t$ yields $\mathrm{App}_T(Z)\le \Delta_{\max}\sum_{t=1}^T e_t(Z)$.

For lagged context, define $\beta_1=\sup_s(1-\max_u P_{su})$ and $\gamma_a=\min_{s\ne s'}|\mu_{s,a}-\mu_{s',a}|$. Consider the decoder that estimates $s_{t-1}$ from $(a_{t-1},r_{t-1})$ by nearest mean and then predicts $s_t$ by the most likely successor of the estimated previous state. Conditional on $a_{t-1}=a$ and $s_{t-1}=s$, a nearest-mean error in favor of a fixed $s'\ne s$ requires
\[
    \eta_{t-1}(\mu_{s',a}-\mu_{s,a})
    \ge \frac{(\mu_{s',a}-\mu_{s,a})^2}{2}.
\]
The left-hand side is $\sigma|\mu_{s',a}-\mu_{s,a}|$-sub-Gaussian, so this event has probability at most $\exp(-\gamma_a^2/(8\sigma^2))$. Union bounding over $s'\ne s$ and then over the realized previous arm gives
\[
 e_t(a_{t-1},r_{t-1})
 \le
 \beta_1+
 \sum_{a\in\mathcal A}\mathbb P(a_{t-1}=a)(S-1)
 \exp\!\left(-\frac{\gamma_a^2}{8\sigma^2}\right).
\]
The $\beta_1$ term is the error of predicting $s_t$ from a correctly decoded $s_{t-1}$.

For synchronous probing over arms $B$, let $m_s^B=(\mu_{s,a})_{a\in B}$ and $\Gamma_B=\min_{s\ne s'}\|m_s^B-m_{s'}^B\|_2$. Suppose the fingerprint observation is $F_t^B=m_{s_t}^B+\xi_t$, where the coordinates of $\xi_t$ are independent and $\sigma$-sub-Gaussian. For a fixed $s'\ne s_t$, nearest-centroid decoding prefers $s'$ to $s_t$ only if
\[
2\langle \xi_t,m_{s'}^B-m_{s_t}^B\rangle
\ge \|m_{s'}^B-m_{s_t}^B\|_2^2.
\]
The projection is $\sigma\|m_{s'}^B-m_{s_t}^B\|_2$-sub-Gaussian, so a one-sided tail bound and a union bound over $s'\ne s_t$ give
\[
 e_t(F_t^B)
 \le (S-1)\exp\!\left(-\frac{\Gamma_B^2}{8\sigma^2}\right).
\]
For an $m$-arm sequential fingerprint $\widetilde F_t^B$, let $\omega_m$ be the probability that the hidden state changes during the probe window. On the no-change event the same nearest-centroid argument applies, while on the change event we use the trivial error bound one. Thus
\[
 e_t(\widetilde F_t^B)
 \le
 \omega_m+(S-1)\exp\!\left(-\frac{\Gamma_B^2}{8\sigma^2}\right).
\]
Under the sticky-state proxy in \Cref{thm:tau}, a union bound over the $m-1$ within-window transitions gives $\omega_m\le (m-1)q$.

\subsection{Proofs of the probe-scheduling bounds}
\label{app:probe-schedule-proof}

\begin{proof}[Proof of \Cref{thm:probe-episode}]
For episode $m$, let $E_m=\{\hat a_m^\ast\ne a^\ast_{s_{c_m}}\}$ denote probe-time optimal-arm error at the completion time $c_m$. For exploit lag $j$, let $D_{m,j}=\{a^\ast_{s_{c_m+j}}\ne a^\ast_{s_{c_m}}\}$ denote optimal-arm drift. If $E_m^c\cap D_{m,j}^c$ holds, then $\hat a_m^\ast=a^\ast_{s_{c_m+j}}$, and the lag-$j$ exploit regret is zero. Otherwise, it is at most $\Delta_{\max}$. Therefore
\[
\Delta_{s_{c_m+j},\hat a_m^\ast}
\le \Delta_{\max}\big(\mathbf{1}\{E_m\}+\mathbf{1}\{D_{m,j}\}\big).
\]
Conditioning on $\mathcal H_{m,j}$ and applying the theorem assumptions gives
\[
\mathbb E\!\left[\Delta_{s_{c_m+j},\hat a_m^\ast}\mid \mathcal H_{m,j}\right]
\le \Delta_{\max}(\varepsilon_m+\nu_j).
\]
Summing this inequality over $j=1,\ldots,L_m-1$ and over $m=1,\ldots,M$, adding the probe-block costs $C_m$, and bounding each leftover round in $\mathcal U$ by $\Delta_{\max}$ proves the claim.
\end{proof}

\begin{proof}[Proof of \Cref{thm:tau}]
A change in the optimal arm over $j$ steps can occur only if at least one of the $j$ Markov transitions changes state. The union bound therefore gives
\[
\mathbb P(a^\ast_{s_{t+j}}\ne a^\ast_{s_t}\mid s_t=s)
\le \sum_{\ell=0}^{j-1} q
= jq.
\]
For complete periodic cycles, apply \Cref{thm:probe-episode} with $L_m=\tau$, $C_m\le\Delta_{\mathrm{probe}}$, and $\varepsilon_m\le\varepsilon_{\mathrm{fp}}$. The average contribution per complete cycle is bounded by
\[
\begin{aligned}
&\Delta_{\mathrm{probe}}
+\Delta_{\max}\sum_{j=1}^{\tau-1}(\varepsilon_{\mathrm{fp}}+jq) \\
&\quad\le
\Delta_{\mathrm{probe}}
+\Delta_{\max}(\tau-1)\varepsilon_{\mathrm{fp}}
+\Delta_{\max}q\frac{\tau(\tau-1)}{2}.
\end{aligned}
\]
Dividing by $\tau$ and using $(\tau-1)/\tau\le1$ yields the three leading terms in \Cref{thm:tau}. If $T$ is not an exact multiple of $\tau$, the endpoint fragment has length at most $\tau$ and contributes the stated lower-order term. The leading terms are minimized at the stated order.
\end{proof}

\begin{corollary}[Hazard-gated schedules]
\label{cor:appendix-hazard}
For a hazard-gated adaptive schedule, define
\[
H_h=\left\lceil\frac{\log(1/(1-\delta_h))}{\lambda_h}\right\rceil.
\]
Suppose each completed-probe episode reuses its fingerprint for at most $H-1$ exploit rounds, each probe block has expected regret at most $c_{\mathrm{probe}}$, and each completed probe has optimal-arm error at most $\bar\varepsilon$. Let
\[
\bar\nu_H=\max_{1\le j\le H-1}\nu_j,
\qquad
\rho_T=\mathbb E[M]/T.
\]
Then
\[
    \frac{1}{T}\mathbb E[R_T]
    \le c_{\mathrm{probe}}\rho_T+
    \Delta_{\max}(\bar\varepsilon+\bar\nu_H)+
    \Delta_{\max}\frac{\mathbb E|\mathcal U|}{T}.
\]
For same-round randomized probing, take $H_{\rm rnd}=\max\{\tau_{\min},H_h\}$. For an $m$-arm sequential probe, use $H_{\rm seq}=H_{\rm rnd}+m-1$, which accounts for probe-window completion.
\end{corollary}

\begin{proof}
The hazard gate satisfies $1-\exp(-\lambda_h a)\ge\delta_h$ exactly when $a\ge H_h$, up to integer rounding. Together with the minimum gap, this gives the stated deterministic upper bound $H$ on the number of post-probe exploit lags for which a completed fingerprint can be reused.

Apply \Cref{thm:probe-episode} with $C_m\le c_{\mathrm{probe}}$, $\varepsilon_m\le\bar\varepsilon$, and $L_m\le H$. Since $\bar\nu_H$ upper-bounds every drift term among the possible exploit lags, the normalized exploit contribution is at most $\Delta_{\max}(\bar\varepsilon+\bar\nu_H)$. The expected number of probe episodes is $\mathbb E[M]=\rho_TT$, so the normalized probe cost is $c_{\mathrm{probe}}\rho_T$.

Adding the normalized leftover term proves the result. The sequential value of $H$ adds $m-1$ rounds because a sequential fingerprint is available only after its probe window completes.
\end{proof}

\end{document}